\documentclass[letterpaper]{article}
\pdfoutput=1

\usepackage{hyperref}
\hypersetup{
	breaklinks=true
	colorlinks=false,
	bookmarksnumbered=false,
	pdfpagemode=UseNone,
	pdfstartpage=1,
	pdfstartview={XYZ null null 1.00},
	pdfpagelayout=OneColumn,
	pdflang=English}

\usepackage{proceed}
\usepackage[margin=1in]{geometry}
\usepackage[latin1]{inputenc}
\usepackage{times}
\usepackage{titlesec}

\usepackage{amsmath}
\usepackage{amsthm}
\usepackage{amssymb}
\usepackage{thmtools}
\usepackage{mathtools}

\usepackage{booktabs}
\usepackage{multicol}

\usepackage{graphicx}
\usepackage{subcaption}

\usepackage{tikz}
\usetikzlibrary{arrows,automata,matrix,calc,patterns,shapes.geometric}
\usepackage{pgfplots}
\pgfplotsset{compat=1.12}

\usepackage[backend=biber,natbib=true,style=authoryear,sorting=nyt,sortcites=true,doi=false,isbn=false,url=false,eprint=true,uniquelist=minyear,mincitenames=1,maxcitenames=2,maxbibnames=9,giveninits=true]{biblatex}
\DeclareNameAlias{sortname}{last-first}
\renewbibmacro{in:}{}
\AtEveryBibitem{%
	\ifentrytype{misc}{%
	}{%
		\clearfield{archivePrefix}%
		\clearfield{arxivId}%
		\clearfield{eprint}%
		\clearfield{primaryClass}%
	}%
}
\defcitealias{Rubenstein2017}{Rubenstein et al. (\bibhyperref{\printdate})}

\usepackage{enumitem}
\newlist{compactitem}{itemize}{3}
\setlist[compactitem]{align=right,labelindent=0.2em,labelwidth=0.8em,labelsep*=0.6em,leftmargin=!,topsep=-3pt,partopsep=0pt,itemsep=2pt,parsep=0pt}
\setlist[compactitem,1]{label=\raisebox{0.0em}{\scalebox{1.0}{$\bullet$}}}
\setlist[compactitem,2]{label=\raisebox{0.1em}{\scalebox{0.5}{$\blacksquare$}}}
\setlist[compactitem,3]{label=\raisebox{0.1em}{\scalebox{0.7}{$\blacktriangleright$}}}
\setlist[compactitem,4]{label=\raisebox{0.1em}{\scalebox{0.6}{$\blacklozenge$}}}
\newlist{compactdesc}{description}{3}
\setlist[compactdesc]{topsep=-3pt,partopsep=0pt,itemsep=2pt,parsep=0pt}
\newlist{compactenum}{enumerate}{3}
\setlist[compactenum]{topsep=-3pt,partopsep=0pt,itemsep=2pt,parsep=0pt}
\setlist[compactenum,1]{label=\arabic*}
\setlist[compactenum,2]{label=\alph*}
\setlist[compactenum,3]{label=\roman*}

\renewcommand\thmcontinues[1]{\emph{Continued}}
\theoremstyle{plain}
\newtheorem{theorem}{Theorem}
\newtheorem{lemma}{Lemma}
\newtheorem{corollary}{Corollary}
\theoremstyle{definition}
\newtheorem{definition}{Definition}

\newtheorem{remark}{Remark}

\newtheorem{examplex}{Example}
\newenvironment{example}
{\pushQED{\qed}\examplex}
{\popQED\endexamplex}
\newtheorem*{example*}{Example}
\newtheorem{proposition}{Proposition}

\title{Beyond Structural Causal Models: Causal Constraints Models}
\author{ \bf{Tineke Blom} \\ Informatics Institute \\ University of Amsterdam \\ The Netherlands \And  {\bf Stephan Bongers} \\ Informatics Institute \\ University of Amsterdam \\ The Netherlands \And {\bf Joris M. Mooij} \\ Informatics Institute \\ University of Amsterdam \\ The Netherlands}

\begin{document}
\maketitle
\vspace*{-10pt}

\begin{abstract}
Structural Causal Models (SCMs) provide a popular causal modeling framework. In this work, we show that SCMs are not flexible enough to give a complete causal representation of dynamical systems at equilibrium. Instead, we propose a generalization of the notion of an SCM, that we call Causal Constraints Model (CCM), and prove that CCMs do capture the causal semantics of such systems. We show how CCMs can be constructed from differential equations and initial conditions and we illustrate our ideas further on a simple but ubiquitous (bio)chemical reaction. Our framework also allows to model functional laws, such as the ideal gas law, in a sensible and intuitive way.
\end{abstract}

\section{INTRODUCTION}
\pagenumbering{arabic}
Real-world processes are often complex and time-evolving. The dynamics of such systems can be modeled by (random) differential equations, which offer a fine-grained description of how the variables in the system change over time. A coarser but more tractable approach are Structural Causal Models (SCMs), which provide a modeling framework that is used in many fields such as biology, the social sciences, and economy \citep{Pearl2000}. Although SCMs have been succesfully applied to certain static systems, a pressing concern is whether SCMs are able to completely model the causal semantics of the stationary behavior of a dynamical system. In this work, we prove that generally SCMs are not flexible enough to completely model dynamical systems at equilibrium. 

We generalize the notion of SCMs and introduce a novel type of causal model, that we call Causal Constraints Models (CCMs). We prove that they give a complete description of the causal semantics of dynamical systems at equilibrium and show how a CCM can be derived from differential equations and initial conditions. We further motivate our approach by pointing out that CCMs, contrary to SCMs, correctly describe the causal semantics of functional laws (e.g. the ideal gas law), which describe relations between variables that are invariant under all interventions. We illustrate the benefits of CCMs on a simple but ubiquitous (bio)chemical reaction.

Causal models that arise from studying the behavior of dynamical systems have received much attention over the years. \citet{Fisher1970, Voortman,Sokol2014,Rubenstein2018, Mogensen2018} consider causal relations in systems that can be modelled by (stochastic) differential equations that are not in equilibrium. In contrast, we consider the stationary behaviour of dynamical systems, which does not require us to model the system's dependence on time. \citet{Mooij2013,Hyttinen2001,Lacerda2008,Mooij2012,Bongers2018} show how cyclic SCMs may arise from studying the stationary behavior of certain dynamical time-series or differential equations, and how in some cases cyclic SCMs can be learned from equilibrium data. SCMs are well-understood and have recently been extended to also include the cyclic case \citep{Forre2017, Bongers2016}. The drawback of the extension in \citet{Forre2017}, with respect to modeling equilibria of dynamical systems, is that it requires the model to have a globally compatible solution under any intervention, which dynamical systems do not, in general, possess. Another modeling approach for dynamical systems at equilibrium is to construct a, possibly cyclic, SCM from the differential equations as \citet{Mooij2013} and \citet{Bongers2018} do. In this work, we show that these approaches to model the causal semantics of the stationary behavior in dynamical systems cannot accomodate the dependence of the equilibria on the initial conditions of the system.

In previous work, researchers have come across subtleties regarding the relation between the causal semantics and conditional independence properties of dynamical systems at equilibrium \citep{Iwasaki1994, Dash, Lacerda2008}. Previously, researchers have made additional assumptions about the underlying dynamical system to circumvent these. Although \citetalias{Rubenstein2017} and \citet{Bongers2018} do not make such restrictions, the price that one pays is that either one must limit the interventions that can be modeled or the equilibrium is no longer uniquely specified and one is limited to modeling the fixed points of the system. To the best of our knowledge, Causal Constraints Models are the first models that can completely capture the causal semantics of the stationary behavior of dynamical systems in general.

A disadvantage of CCMs over SCMs is that they do not yet possess the intuitive graphical interpretation that SCMs have. We consider representations of the independence structure of CCMs outside the scope of this work.

\subsection{STRUCTURAL CAUSAL MODELS}
A statistical model over random variables, taking value in a measurable space $\boldsymbol{\mathcal{X}}$, usually is a pair $(\boldsymbol{\mathcal{X}}, \boldsymbol{\mathbb{P}}^{\boldsymbol{\mathcal{X}}})$ where $\boldsymbol{\mathbb{P}}^{\boldsymbol{\mathcal{X}}}$ is a (parametrized) family of probability distributions on $\boldsymbol{\mathcal{X}}$. A causal model on the other hand, can be thought of as a family of statistical models, one for each (perfect) intervention,
\begin{equation}
\bar{\mathbb{P}}^{\boldsymbol{\mathcal{X}}} =
  \left(\mathbb{P}^{\boldsymbol{\mathcal{X}}}_{\text{do}(I, \boldsymbol{\xi}_I)}: \quad I\in\mathcal{P}(\mathcal{I}), \quad \boldsymbol{\xi}_I \in \boldsymbol{\mathcal{X}}_I\right),
\end{equation}
where $\mathcal{I}$ is an index set and $\mathcal{P}(\mathcal{I})$ denotes the power set of $\mathcal{I}$ (i.e.\ the set of all subsets of $\mathcal{I}$). $I$ represents the intervention target and $\boldsymbol{\xi}_I$ a tuple of intervention values. The null intervention $\mathrm{do}(\emptyset)$ for $I=\emptyset$ corresponds to the observed system.

SCMs are a special type of causal models that are specified by structural equations. Our formal treatment of SCMs mostly follows \citet{Bongers2016, Pearl2000}. For the purposes of this paper, we deviate from the usual definition of SCMs by not assuming independence of exogenous variables and by not requiring acyclicity (i.e. recursiveness).

\begin{definition}
\label{def:SCM}
Let $\mathcal{I}$ and $\mathcal{J}$ be index sets. A \emph{Structural Causal Model (SCM)} $\mathcal{M}$ is a triple $(\boldsymbol{\mathcal{X}}, F, \boldsymbol{E})$, with:
\begin{compactitem}
\item a product of standard measurable spaces $\boldsymbol{\mathcal{X}}=\prod_{i\in\mathcal{I}}\mathcal{X}_i$ (domains of endogenous variables),
\item a tuple of exogenous random variables $\boldsymbol{E}=(E_j)_{j\in\mathcal{J}}$ taking value in a product of standard measurable spaces $\boldsymbol{\mathcal{E}}=\prod_{j\in\mathcal{J}}\mathcal{E}_j$,
\item a family $F$ of measurable functions:\footnote{$\text{pa}(i)\subseteq \mathcal{I} \cup \mathcal{J}$ denotes a subset of indexes that are sufficient to determine the values of $f_i$.}
\begin{equation*}
f_i: \boldsymbol{\mathcal{X}}_{\text{pa}(i)\cap\mathcal{I}}\times \boldsymbol{\mathcal{E}}_{\text{pa}(i)\cap\mathcal{J}} \to \mathcal{X}_i,
  \quad \forall i \in \mathcal{I}.
\end{equation*}
\end{compactitem}
\end{definition}

Note that a cyclic structural causal model does not need to imply a unique joint distribution $\mathbb{P}^{\boldsymbol{\mathcal{X}}}_{\text{do}(\emptyset)}$ on the space of endogenous variables in the observed system, although acyclic SCMs do \citep{Bongers2016}. When there exists a unique solution $\boldsymbol{x}\in\boldsymbol{\mathcal{X}}$ to the \emph{structural equations} 
\begin{align*}
x_i = f_{i}(\boldsymbol{x}_{\text{pa}(i)\cap\mathcal{I}}, \boldsymbol{e}_{\text{pa}(i)\cap\mathcal{J}}),
  \quad \forall i\in \mathcal{I}
\end{align*}
for almost all $\boldsymbol{e}\in\boldsymbol{\mathcal{E}}$, we say that the model is uniquely solvable. 

\begin{definition}
\label{def:solution of SCM}
We say that a random variable $\boldsymbol{X}=(X_i)_{i\in\mathcal{I}}$ is a \emph{solution} to an SCM $\mathcal{M}=(\boldsymbol{\mathcal{X}}, F, \boldsymbol{E})$ if
\begin{equation*}
X_i = f_{i}(\boldsymbol{X}_{\text{pa}(i)\cap\mathcal{I}}, \boldsymbol{E}_{\text{pa}(i)\cap\mathcal{J}})
  \quad \text{a.s.},
  \quad \forall i\in \mathcal{I}.
\end{equation*}
\end{definition}
An SCM may have a unique (up to zero sets) solution, multiple solutions, or there may not exist any solution at all.

There are many types of interventions, corresponding to different experimental procedures, that can be modeled in an SCM. For the remainder of this work, we consider \emph{perfect} (also known as ``surgical'' or ``atomic'') interventions that force variables to take on a specific value through some external force acting on the system.

\begin{definition}
A \emph{perfect intervention} $\text{do}(I,\boldsymbol{\xi}_I)$ with target $I\subseteq\mathcal{I}$ and value $\boldsymbol{\xi}_I\in\boldsymbol{\mathcal{X}}_I$ on an SCM $\mathcal{M}=(\boldsymbol{\mathcal{X}}, F, \boldsymbol{E})$ maps it to the \emph{intervened} SCM $\mathcal{M}_{\text{do}(I,\boldsymbol{\xi}_I)}=(\boldsymbol{\mathcal{X}}, \widetilde{F},\boldsymbol{E})$ with $\widetilde{F}$ the family of measurable functions:
\begin{multline*}
\tilde{f}_i(\boldsymbol{x}_{\text{pa}(i)\cap\mathcal{I}}, \boldsymbol{e}_{\text{pa}(i)\cap\mathcal{J}})\\
 = \begin{cases}
  \xi_i & i\in I,\\
  f_i(\boldsymbol{x}_{\text{pa}(i)\cap\mathcal{I}}, \boldsymbol{e}_{\text{pa}(i)\cap\mathcal{J}}) & i\in\mathcal{I}\backslash I.
\end{cases}
\end{multline*}
\end{definition}

Note that the solvability of an SCM may change after a perfect intervention, e.g.\ a uniquely solvable SCM may no longer be so after certain interventions.

\subsection{DYNAMICAL SYSTEMS}
\label{sec:preliminaries:dynamical systems}
We consider dynamical systems $\mathcal{D}$ describing $p = |\mathcal{I}|$ (random) variables $\boldsymbol{X}(t)$ taking value in $\boldsymbol{\mathcal{X}}=\mathbb{R}^p$. They consist of a set of coupled first-order ordinary differential equations (ODEs) where the initial conditions $\boldsymbol{X}(0)$ are determined by exogenous random variables $\boldsymbol{E}=(E_i)_{i\in\mathcal{I}}$ taking value in $\boldsymbol{\mathcal{E}}=\mathbb{R}^p$. That is,
\begin{align*}
\dot{X}_i(t)
  &= f_i(\boldsymbol{X}(t)), &\forall i\in\mathcal{I},\\
X_i(0)
  &=E_i, &\forall i\in\mathcal{I},
\end{align*}
where the $f_i$ are locally Lipschitz continuous functions.\footnote{If the dynamics depends on (random) parameters, they can be modeled as additional endogenous variables with vanishing time derivatives and initial conditions corresponding to the (random) parameters. Therefore, without loss of generality, we may assume that the functions $f_i$ only depend on $\boldsymbol{X}$.} Throughout this paper, we will assume for any dynamical system we encounter that for $\mathbb{P}^{\boldsymbol{E}}$-almost every $\boldsymbol{e}\in\mathbb{R}^{p}$ the initial value problem with $\boldsymbol{X}(0) = \boldsymbol{e}$ has a unique solution $\boldsymbol{X}(t, \boldsymbol{e})$ for all $t \ge 0$, given by
\begin{equation}\label{eq:ODE_solution}
\boldsymbol{X}(t, \boldsymbol{e}) = \boldsymbol{X}(0,\boldsymbol{e}) + \int_{0}^{t} \boldsymbol{f}\big(\boldsymbol{X}(s, \boldsymbol{e})\big) ds.
\end{equation}
This solution $\boldsymbol{X}(t, \boldsymbol{e})$ can be trivially extended to $\boldsymbol{\mathcal{E}}$ and it is measurable in $\boldsymbol{e}$ for all $t$ \citep{Han2017}.

A \emph{fixed point} (or equilibrium point) of $\mathcal{D}$ is a point $\boldsymbol{x}^*\in\mathbb{R}^p$ for which $\boldsymbol{f}(\boldsymbol{x}^*)=\boldsymbol{0}$. For $\boldsymbol{e}\in\mathbb{R}^p$, the dynamical system converges to an equilibrium $\boldsymbol{X}^*(\boldsymbol{e})\in\mathbb{R}^p$ if
\begin{equation}
\label{eq:convergence ivp}
\lim_{t\to\infty} \boldsymbol{X}(t, \boldsymbol{e}) = \boldsymbol{X}^*(\boldsymbol{e}).
\end{equation}
If for $\mathbb{P}^{\boldsymbol{E}}$-almost every $\boldsymbol{e}$ the limit in equation (\ref{eq:convergence ivp}) exists, then we say that $\mathcal{D}$ converges to the equilibrium solution $\boldsymbol{X}^*=\lim_{t\to\infty} \boldsymbol{X}(t,\boldsymbol{E})$.

Interventions on dynamical systems can be modeled in different ways. One could for example fix the value of targeted values at one time-point. Alternatively, one could fix the trajectory of the targeted values as in \citet{Rubenstein2018}. Here, we follow \citet{Mooij2013} and define interventions as operations that fix the value of the targeted variables to a constant (for all time). 

\begin{definition}
A \emph{perfect intervention} $\text{do}(I,\boldsymbol{\xi}_I)$ where $I\subseteq\mathcal{I}$ and $\boldsymbol{\xi}_I\in\boldsymbol{\mathcal{X}}_I$ results in the \emph{intervened} dynamical system $\mathcal{D}_{\text{do}(I,\boldsymbol{\xi}_I)}$ specified by
\begin{align*}
\dot{X}_i(t) &= 0, && X_i(0)=\xi_i, &&\forall i\in I,\\
\dot{X}_i(t) &= f_i(\boldsymbol{X}(t)), &&X_i(0)=E_i, &&\forall i\in\mathcal{I}\backslash I.
\end{align*}
\end{definition}

We say that a causal model $\mathcal{M}$ \emph{completely} captures the causal semantics of the stationary behaviour of a dynamical system $\mathcal{D}$ if for all $I\subseteq\mathcal{I}$ and all $\boldsymbol{\xi}_I\in\boldsymbol{\mathcal{X}}_I$: the equilibrium solutions of $\mathcal{D}_{\mathrm{do}(I,\boldsymbol{\xi}_I)}$ coincide with the solutions of $\mathcal{M}_{\mathrm{do}(I,\boldsymbol{\xi}_I)}$ (up to $\mathbb{P}^{\boldsymbol{E}}$-null sets).

The construction of SCMs from dynamical systems in \citet{Mooij2013} relies on the fact that for systems that converge to a fixed point independent of initial conditions (i.e.\ globally asymptotically stable systems), the fixed point directly gives a complete description of its stationary behavior. A much weaker stability assumption is \emph{(global) semistability} \citep{Campbell1979,Bhat1999}, where solutions of a system converge to a stable equilibrium determined by initial conditions. Our definition follows \citet{Haddad2010}.

\begin{definition}
\label{def:semistability}
Let $\mathcal{D}$ be a dynamical system and $\boldsymbol{\mathcal{U}}\subseteq\mathbb{R}^p$ an invariant subset (i.e. if $\boldsymbol{x}(0)\in\boldsymbol{\mathcal{U}}$ then $\boldsymbol{x}(t)\in\boldsymbol{\mathcal{U}}$ for all $t\geq 0$).
A fixed point $\boldsymbol{x}^*\in\boldsymbol{\mathcal{U}}$ is Lyapunov stable with respect to $\boldsymbol{\mathcal{U}}$ if for all $\boldsymbol{x}(0)\in\boldsymbol{\mathcal{U}}$: for all $\epsilon>0$ there exists $\delta>0$ such that if $\lVert \boldsymbol{x}(0)-\boldsymbol{x}^* \rVert < \delta$ then  for all $t\geq 0$, $\lVert \boldsymbol{x}(t)-\boldsymbol{x}^* \rVert < \epsilon$. It is \emph{semistable} w.r.t. $\boldsymbol{\mathcal{U}}$ if, additionally, there exists a relatively open subset\footnote{$\boldsymbol{\mathcal{N}}$ is a relatively open subset of $\boldsymbol{\mathcal{U}}$ if there is an open set $\boldsymbol{\mathcal{N}}'\subseteq\mathbb{R}^p$ such that $\boldsymbol{\mathcal{N}} = \boldsymbol{\mathcal{N}}'\cap\boldsymbol{\mathcal{U}}$.} $\boldsymbol{\mathcal{N}}$ of $\boldsymbol{\mathcal{U}}$ that contains $\boldsymbol{x}^*$ such that $\boldsymbol{x}(t)$ converges to a Lyapunov stable fixed point for all $\boldsymbol{x}(0)\in\boldsymbol{\mathcal{N}}$. If $\boldsymbol{\mathcal{N}}=\boldsymbol{\mathcal{U}}$ then $\boldsymbol{x}^*$ is globally semistable w.r.t. $\boldsymbol{\mathcal{U}}$. Finally, we say that $\mathcal{D}$ is globally semistable w.r.t. $\boldsymbol{\mathcal{U}}$ if all its fixed points are globally semistable w.r.t. $\boldsymbol{\mathcal{U}}$.
\end{definition}

\begin{definition}
	\label{def:structurally semistable}
  A dynamical system $\mathcal{D}$ is \emph{structurally semistable} if for all $I\subseteq\mathcal{I}$ there exists $\boldsymbol{\mathcal{U}}\subseteq\mathbb{R}^p$ with $\mathbb{P}^{\boldsymbol{E}_{\mathcal{I}\backslash I}}( \boldsymbol{\mathcal{U}}_{\mathcal{I}\backslash I})=1$ such that: $\mathcal{D}_{\mathrm{do}(I,\boldsymbol{\xi}_I)}$ is globally semistable w.r.t. $\boldsymbol{\mathcal{U}}$ (for any $\boldsymbol{\xi}_I\in\boldsymbol{\mathcal{X}}_I$).
\end{definition}

Whether a dynamical system converges to a certain fixed point depends on initial conditions. This dependence can often be described by constants of motion, and there exists a vast literature on how and when these can be derived from differential equations. The notion of semistability is appropriate in many real-world applications in chemical kinetics, environmental, and economic systems \citep{Haddad2010}. For chemical reaction networks, there exist convenient criteria on the network structure that guarantee global semistability \citep{Chellaboina2009}, and for mechanical systems semistability characterizes the motion of rigid bodies subject to damping \citep{Bhat1999}.

\section{DYNAMICAL SYSTEMS AS SCMs}
We consider SCM representations of the equilibria in a chemical reaction and conclude that, generally, SCMs are not flexible enough to completely capture the causal semantics of stationary behaviour in dynamical systems.

\subsection{BASIC ENZYME REACTION}
The basic enzyme reaction is a well-known example of a (bio)chemical reaction network. It describes a system  where a substrate $S$ reacts with an enzyme $E$ to form a complex $C$ which is then converted into a product $P$ and the enzyme \citep{Murray2002}. In the open enzyme reaction a constant influx of substrate and an efflux of product are added \citep{Belgacem2012}. The process can be presented by the following reaction graph,
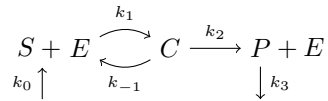
\begin{minipage}{\linewidth}
\centering
\begin{tikzpicture}[auto]
\centering
\node (S) {$S + {}$};
\node (E1) [right of=S, node distance=5mm] {$E$};
\node [below of=S, node distance=8mm] (source) {};
\node [right of=E1, node distance=12mm] (SE) {$C$};
\node [right of=SE, node distance=12mm] (P) {$P$};
\node [right of=P, node distance=5mm] (E2) {${} + E$};
\node [below of=P, node distance=8mm] (sink) {};
\draw[->] (E1) edge[bend left] node {\tiny$k_1$} (SE);
\draw[->] (SE) edge[bend left] node {\tiny$k_{-1}$} (E1);
\draw[->] (SE) to node {\tiny$k_2$} (P);
\draw[->] (P) to node {\tiny$k_3$} (sink);
\draw[->] (source) to node {\tiny$k_0$} (S);
\end{tikzpicture}
\end{minipage}
and $\boldsymbol{k}=[k_0,k_{-1},k_1,k_2,k_3]$ strictly positive parameters.

Differential equations for the concentrations of each molecule in the system can be obtained by application of the law of mass-action, which states that the rate of a reaction is proportional to the product of the concentration of the reactants \citep{Murray2002}, yielding:
\begin{align}
\label{eq:ber ode first}
\dot{S}(t) &= k_0-k_{1}S(t)E(t) + k_{-1} C(t),\\
\label{eq:ber ode second}
\dot{E}(t) &= -k_1S(t)E(t) + (k_{-1}+k_2)C(t),\\
\dot{C}(t) &= k_1S(t)E(t) - (k_{-1}+k_2)C(t),\\
\label{eq:ber ode last}
\dot{P}(t) &= k_2C(t) - k_3 P(t),\\
&(S(0),E(0),C(0),P(0)) = (s_0,e_0,c_0,p_0).
\end{align}

We simulated the system in (\ref{eq:ber ode first}) to (\ref{eq:ber ode last}) with random initial conditions and also under interventions on $S$ and $E$. Figures \ref{fig:ber observational} to \ref{fig:ber interventional E} show how the time-trajectories of the concentrations depend on initial conditions in different interventional settings.

\definecolor{mycolor1}{rgb}{0.00000,0.75000,0.75000}%
\definecolor{mycolor2}{rgb}{0.75000,0.00000,0.75000}%
\definecolor{mycolor3}{rgb}{0.75000,0.75000,0.00000}%

\begin{figure*}[t]
\centering
\begin{subfigure}[t]{0.31\textwidth}
\centering
\begin{tikzpicture}
\begin{axis}[
width=1.18\textwidth,
height=0.9\textwidth,
xlabel={\footnotesize $t$},
xmin=0,xmax=35,xtick={0,7,...,35},
xticklabels={},
xlabel shift=-5pt,
ylabel={\footnotesize $\log$ concentration},
ymin=-2,ymax=3,ytick={-2,-1,0,1,2,3},
yticklabels={},
ylabel shift=-5pt,
legend pos=north east,
legend columns=2]

\foreach \F in {1,2,...,10}{
	\addplot[mycolor1] table[x index = {0}, y index = {1}, col sep = comma] {ber_observation_data\F.dat};
	\addplot[mycolor2] table[x index = {0}, y index = {2}, col sep = comma] {ber_observation_data\F.dat};
	\addplot[blue!90!white] table[x index = {0}, y index = {3}, col sep = comma] {ber_observation_data\F.dat};
	\addplot[mycolor3] table[x index = {0}, y index = {4}, col sep = comma] {ber_observation_data\F.dat};
}

\addlegendentry{\tiny $S(t)\;\;$}
\addlegendentry{\tiny $E(t)$}
\addlegendentry{\tiny $C(t)\;\;$}
\addlegendentry{\tiny $P(t)$}
\end{axis}
\end{tikzpicture}
\caption{$S$ and $P$ converge to an equilibrium that depends on initial conditions in the observed system.}
\label{fig:ber observational}
\end{subfigure}%
\hspace{0.0349\textwidth}%
\begin{subfigure}[t]{0.31\textwidth}
\centering
\begin{tikzpicture}
\begin{axis}[
width=1.18\textwidth,
height=0.9\textwidth,
xlabel={\footnotesize $t$},
xmin=0,xmax=10,xtick={0,2,...,10},
xticklabels={},
xlabel shift=-5pt,
ylabel={\footnotesize $\log$ concentration},
ymin=-2,ymax=2,ytick={-2,-1,0,1,2},
yticklabels={},
ylabel shift=-5pt,
legend pos=south east,
legend columns=2]

\foreach \F in {1,2,...,10}{
	\addplot[mycolor1] table[x index = {0}, y index = {1}, col sep = comma] {ber_intervention_data_S\F.dat};
	\addplot[mycolor2] table[x index = {0}, y index = {2}, col sep = comma] {ber_intervention_data_S\F.dat};
	\addplot[blue!90!white] table[x index = {0}, y index = {3}, col sep = comma] {ber_intervention_data_S\F.dat};
	\addplot[mycolor3] table[x index = {0}, y index = {4}, col sep = comma] {ber_intervention_data_S\F.dat};
}

\addlegendentry{\tiny $S(t)\;\;$}
\addlegendentry{\tiny $E(t)$}
\addlegendentry{\tiny $C(t)\;\;$}
\addlegendentry{\tiny $P(t)$}
\end{axis}
\end{tikzpicture}
\caption{$C,E,$ and $P$ converge to an equilibrium that depends on the initial conditions after an intervention on $S$.}
\label{fig:ber interventional S}
\end{subfigure}%
\hspace{0.0349\textwidth}%
\begin{subfigure}[t]{0.31\textwidth}
\centering
\begin{tikzpicture}
\begin{axis}[
width=1.18\textwidth,
height=0.9\textwidth,
xlabel={\footnotesize $t$},
xmin=0,xmax=20,xtick={0,4,...,20},
xticklabels={},
xlabel shift=-5pt,
ylabel={\footnotesize $\log$ concentration},
ymin=-2,ymax=2,ytick={-2,-1,0,1,2},
yticklabels={},
ylabel shift=-5pt,
legend pos=north east,
legend columns=2]

\foreach \F in {1,2,...,10}{
	\addplot[mycolor1] table[x index = {0}, y index = {1}, col sep = comma] {ber_intervention_data_E\F.dat};
	\addplot[mycolor2] table[x index = {0}, y index = {2}, col sep = comma] {ber_intervention_data_E\F.dat};
	\addplot[blue!90!white] table[x index = {0}, y index = {3}, col sep = comma] {ber_intervention_data_E\F.dat};
	\addplot[mycolor3] table[x index = {0}, y index = {4}, col sep = comma] {ber_intervention_data_E\F.dat};
}

\addlegendentry{\tiny $S(t)\;\;$}
\addlegendentry{\tiny $E(t)$}
\addlegendentry{\tiny $C(t)\;\;$}
\addlegendentry{\tiny $P(t)$}
\end{axis}
\end{tikzpicture}
\caption{$S,C,$ and $P$ converge to an equilibrium that is independent of the initial conditions after an intervention on $E$.}
\label{fig:ber interventional E}
\end{subfigure}
\caption{Temporal dependence of concentrations in the basic enzyme reaction in  (\ref{eq:ber ode first}) to (\ref{eq:ber ode last}) with random initial conditions and $\boldsymbol{k}=[0.4,0.3,1.0,1.1,0.5]$. Other choices for the rate parameters give qualitatively similar results.}
\label{fig:xxx}
\end{figure*}
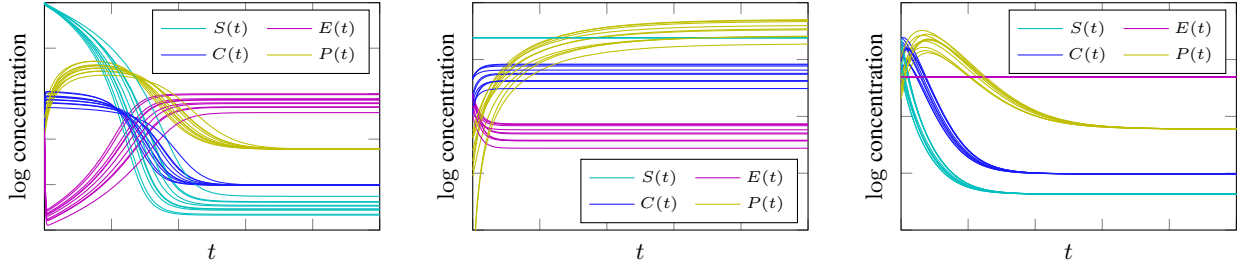

\subsubsection{EQUILIBRIUM SOLUTIONS}
\label{sec:ber equilibrium states}

By explicit calculation one can verify that given strictly positive initial conditions, the dynamical system converges to an equilibrium $(S^*,C^*,E^*,P^*)$ if it exists, for any perfect intervention (one can also check that the system is structurally semistable).\footnote{See \citet{Belgacem2012} and supplementary material for details.} The equilibria can be found by deriving constraints on solutions of the system:
\begin{compactitem}
\item At equilibrium the system is at rest and all time derivatives (in the equations of motion) must vanish. The equation of motion of each variable then results in a constraint that is invariant under all interventions that do not target that variable. For example, equation (\refeq{eq:ber ode first}) yields the equilibrium equation
\begin{equation*}
k_0-k_{1}S^*E^* + k_{-1} C^* = 0,
\end{equation*}
which constrains the equilibrium state unless $S$ is targeted by an intervention.

\item Symmetries or (linear) dependencies between the time derivatives lead to conservation laws (i.e.\ constants of motion), which are relations between variables that are time-invariant but that are typically invariant under fewer interventions than constraints of the first type. For example, since $\dot{C}(t)+\dot{E}(t)=0$ for all $t$, we have that
\begin{equation}
\label{eq:ber constant of motion}
C(t)+E(t) = c_0 + e_0, \qquad \forall\, t,
\end{equation}
unless $C$, $E$ or both $C$ and $E$ are targeted by an intervention.

\item A system may contain (derived) variables whose time-derivative does not depend on itself. Since
\begin{equation}
\dot{S}(t) - \dot{E}(t) = k_0 - k_2 C(t),
\end{equation}
the variable $C$ cannot be `freely manipulated', in the sense that $S(t)-E(t)$ does not converge to equilibrium under interventions $\text{do}(C=\xi_C)$ when $\xi_C\neq\frac{k_0}{k_2}$. For $\xi_C=\frac{k_0}{k_2}$ a new constant of motion is introduced so that $S(t)-E(t)=s_0-e_0$ unless $S$, $E$ or both $S$ and $E$ are targeted by an intervention.
\end{compactitem}

It can be shown, through explicit calculation, that for any perfect intervention these constraints have no solution when the dynamical system does not converge to an equilibrium and they have a unique solution when the system does converge to an equilibrium. A complete causal description of the system can be found in Table \ref{tab:ber complete solutions} in the supplementary material. Table \ref{tab:ber solutions} and Figure \ref{fig:xxx} illustrate the rich causal semantics of this system (e.g. an intervention on $S$ makes $C^*$ dependent on the initial conditions, while an intervention on $E$ makes $S^*$ independent of the initial conditions).

\subsection{SCM REPRESENTATION}
Globally asymptotically stable dynamical systems converge to a unique fixed point and \citet{Mooij2013} show how SCMs can then be constructed from ordinary differential equations. For the basic enzyme reaction (which is not globally asymptotically stable) their construction method would yield the structural equations:
\begin{align}
\label{eq:ber se s}
\textstyle S^* &= \textstyle \frac{k_0 + k_{-1}C^*}{k_1E^*},\\
\label{eq:ber se e}
\textstyle E^* &= \textstyle \frac{(k_{-1}+k_2)C^*}{k_1S^*},\\
\label{eq:ber se c}
\textstyle C^* &= \textstyle \frac{k_1S^*E^*}{k_{-1}+k_2},\\
\label{eq:ber se p}
\textstyle P^* &= \textstyle \frac{k_2}{k_3}C^*.
\end{align}
While this SCM represents the causal semantics of the system's fixed points, it would be underspecified as an SCM for the stationary behavior of the basic enzyme reaction. Indeed, this SCM has multiple solutions, corresponding to different possible initial conditions of the dynamical system and it does not contain any information on which of its solutions is realized. Theorem \ref{thm:no scm representation} shows that a complete SCM representation of the stationary behavior in the basic enzyme reaction does not exist.

\begin{theorem}
\label{thm:no scm representation}
The causal semantics of the stationary behaviour of the basic enzyme reaction, and its dependence on initial states, cannot be completely represented by an SCM with endogenous variables $S^*,E^*,C^*,P^*$.
\end{theorem}
\begin{proof}
The system converges to an equilibrium under the intervention $\text{do}(E^*=e, C^*=c, P^*=p)$.\footnote{See supplementary material A2.} Setting $\dot{S}=0$ in (\ref{eq:ber ode first}) and solving for $S$, we find that $S^*=\frac{k_0 +k_{-1} c}{k_1 e}$ and therefore any SCM that models the effect of this intervention correctly must have a structural equation that is equivalent to equation (\ref{eq:ber se s}). Analogously, considering the remaining three interventions on three out of four variables, we find that an SCM that correctly models the effects of those interventions must have structural equations for $E^*,C^*$ and $P^*$ that are equivalent to the structural equations (\ref{eq:ber se e}) to (\ref{eq:ber se p}), respectively. Table \ref{tab:ber solutions} shows that the system converges to an equilibrium that depends on the initial conditions $c_0$ and $e_0$ under the null intervention. This equilibrium is a solution of the structural equations in (\ref{eq:ber se s}) to (\ref{eq:ber se p}). However, these structural equations do not depend on initial conditions and admit other solutions as well. Therefore they do not completely represent the stationary behaviour of the system.
\end{proof}

\begin{table}
\caption{Equilibrium solutions to the intervened dynamical system of the basic enzyme reaction in (\ref{eq:ber ode first}) to (\ref{eq:ber ode last}) under various interventions, where $y=\frac{1}{2}\sqrt{(e_0-s_0)^2 +  4\frac{k_0(k_{-1}+k_2)}{k_1k_2}}$.}
\label{tab:ber solutions}
\resizebox{\linewidth}{!}{
\begin{tabular}{l|cccc}
$I$ & $S^*$ & $C^*$ & $E^*$\\[0.2ex]
\hline
$\emptyset$       & $\frac{k_0 + k_{-1} \frac{k_0}{k_2}}{k_1 (e_0 + c_0 - \frac{k_0}{k_2})}$ & $\frac{k_0}{k_2}$ & $e_0 + c_0 - \frac{k_0}{k_2}$ \\[0.2ex]
$S=\xi_s$  & $\xi_s$ & $\frac{k_1 \xi_s (e_0 + c_0)}{k_{-1} + k_2 + k_1 \xi_s} $ & $\frac{(k_{-1} + k_2)(e_0 + c_0)}{k_{-1} + k_2 + k_1 \xi_s} $ \\[0.2ex]
$C=\frac{k_0}{k_2}$   & $\frac{(e_0-s_0)}{2} + y$ & $\frac{k_0}{k_2}$ & $\frac{-(e_0-s_0)}{2} + y$ \\[0.2ex]
$E=\xi_e$ & $\frac{k_0 + k_{-1} \frac{k_0}{k_2}}{k_1 \xi_e}$ & $\frac{k_0}{k_2}$ & $\xi_e$ \\
\end{tabular}}
\end{table}

\section{CAUSAL CONSTRAINTS MODELS}
We introduce Causal Constraints Models (CCMs) and prove that they completely capture the causal semantics of the stationary behaviour of dynamical systems.

SCMs are specified by structural equations which constrain its solutions unless the corresponding variable is targeted by an intervention. CCMs are specified by \emph{causal constraints}: relations between variables that constrain the solutions of the model under explicitly specified intervention targets. 

\begin{definition}
\label{def:ccm}
Let $\mathcal{I}$, $\mathcal{J}$ and $\mathcal{K}$ be index sets. A \emph{Causal Constraints Model (CCM)} is a triple $(\boldsymbol{\mathcal{X}},\Phi,\boldsymbol{E})$, with:
\begin{compactitem}
\item $\boldsymbol{\mathcal{X}}$ and $\boldsymbol{E}$ as in Definition \ref{def:SCM} (domain of endogenous variables and tuple of exogenous random variables respectively),
\item a set $\Phi=\{\phi_k:k\in\mathcal{K}\}$ of causal constraints, each of which is a triple $\phi_k=(f_k,c_k,A_k)$ where,
\begin{compactitem}[leftmargin=0.35cm]
\item $f_k:\boldsymbol{\mathcal{X}}_{\mathrm{pa}(k)\cap\mathcal{I}}\times\boldsymbol{\mathcal{E}}_{\mathrm{pa}(k)\cap\mathcal{J}}\to\mathcal{Y}_k$ is a measurable function, $\mathcal{Y}_k$ a standard measurable space and $\mathrm{pa}(k) \subseteq \mathcal{I} \cup \mathcal{J}$,
\item $c_k\in \mathcal{Y}_k$ is a constant,
\item $A_k\subseteq\mathcal{P}(\mathcal{I})$ specifies the set of intervention targets under which $\phi_k$ is \emph{active}.
\end{compactitem}
\end{compactitem}
\end{definition}

The following example illustrates the interpretation of causal constraints in CCMs.

\begin{example}[label=ex:price supply demand]
	Consider the price, supply, and demand of a certain product, denoted by $P,S,$ and $D$ respectively, related by the following causal constraint:
	\begin{equation}
	\label{eq:price supply demand}
	(f,c,A) = \left(S-D,\, 0, \, \{\emptyset, \{D\}, \{S\}, \{D,S\}\}\right).
	\end{equation}
	The constraint $S-D=0$ is active in the observational setting because $\emptyset\in A$. It is also active when either $D$, $S$ or both $D$ and $S$ are targeted by an intervention. The constraint becomes inactive after an intervention on $P$. In other words, supply equals demand unless the price of the product is intervened upon (e.g.\ price-fixing). 
\end{example}

\subsection{CCM SOLUTIONS}

We define a solution of a CCM in complete analogy with the definition of a solution of an SCM.
\begin{definition}
\label{def:rv solutions}
Let $\mathcal{M}=(\boldsymbol{\mathcal{X}}, \Phi, \boldsymbol{E})$ be a CCM and let $\Phi^{\emptyset} := \{\phi_k=(f_k,c_k,A_k) \in \Phi : \emptyset \in A_k\}$. A random variable $\boldsymbol{X}$ taking value in $\boldsymbol{\mathcal{X}}$ is a \emph{solution} of $\mathcal{M}$ if
\begin{equation*}
f_{k}(\boldsymbol{X}_{\mathrm{pa}(k)\cap\mathcal{I}},\boldsymbol{E}_{\mathrm{pa}(k)\cap\mathcal{J}})=c_{k}\, \mathrm{a.s.},
  \, \forall (f_k,c_k,A_k)\in\Phi^{\emptyset}.
\end{equation*}

\end{definition}
Similar to SCMs, a CCM has either no solution, or it has a solution and all its solutions may either induce a unique or multiple distributions.

\subsection{CCM INTERVENTIONS}
Interventions on SCMs act on its structural equations. Analogously, an intervention on a CCM acts on its causal constraints. Roughly speaking, the activation sets of the causal constraints in the model are updated and additional causal constraints describe the intervention.

\begin{definition}
\label{def:interventions}
Let $\mathcal{M}=(\boldsymbol{\mathcal{X}}, \Phi, \boldsymbol{E})$ be a CCM and let $I\subseteq\mathcal{I}$ be the intervention target and $\boldsymbol{\xi}_I\in\boldsymbol{\mathcal{X}}_I$ the target value. The \emph{intervened CCM} is given by $\mathcal{M}_{\text{do}(I,\boldsymbol{\xi}_I)}=(\boldsymbol{\mathcal{X}}, \widetilde{\Phi}, \boldsymbol{E})$ where:
\begin{compactitem}
\item for each $i\in I$ we add a causal constraint describing the intervened value of the targets, $\left(x_i, \xi_i, \mathcal{P}\left(\mathcal{I}\backslash \{i\}\right)\right)\in \widetilde{\Phi}$,
\item for each causal constraint $(f,c,A)\in \Phi$ we get a modified causal constraint $(f,c,A_{\text{do}(I)})\in \widetilde{\Phi}$ if $A_{\text{do}(I)}\neq\emptyset$, where
\begin{align*}
A_{\text{do}(I)} &= \{A_i \setminus J : A_i \in A, J \subseteq I \subseteq A_i \}.
\end{align*}
\end{compactitem}
\end{definition}

Definition \ref{def:interventions} says that for any $A_i \in A$, and for any combination of two subsequent interventions such that $I_1 \cup I_2 = A_i$, the constraint will be active. So after $I_1$ (which needs to be a subset of $A_i$), any $I_2$ that adds the remaining elements $A_i \setminus I_1$ (plus possibly any elements that were already in $I_1$) will activate the constraint.

\begin{example*}
The effect of different interventions on a set $A_{\text{do}(\emptyset)} = \{\emptyset, \{1,2\}, \{2,3\}\}$:
\begin{align*}
A_{\text{do}(1)} &= \{\{2\},\{1,2\}\},\\
A_{\text{do}(2)} &= \{\{1\}, \{1,2\},\{3\}, \{2,3\}\},\\
A_{\text{do}(\{1,2\})} &= A_{\text{do}(1)\text{do}(2)} =  A_{\text{do}(2)\text{do}(1)}\\
&= \{\emptyset,\{1\},\{2\},\{1,2\}\},\\
A_{\text{do}(\{1,2,3\})} &= \emptyset.
\end{align*}
\end{example*}

Lemma \ref{lemma:interventions commute} shows that the effect of multiple interventions on a CCM depends neither on whether the interventions are performed simultaneously or sequentially nor on the order in which they are performed.
\begin{lemma}
\label{lemma:interventions commute}
Let $\mathcal{M}$ be a CCM for variables indexed by $\mathcal{I}$ and let $I,J\subseteq \mathcal{I}$ be two disjoint sets of intervention targets with intervention values $\boldsymbol{\xi}_I\in\boldsymbol{\mathcal{X}}_I$ and $\boldsymbol{\xi}_J\in\boldsymbol{\mathcal{X}}_J$ respectively. Then
\begin{align*}
\bigl(\mathcal{M}_{\mathrm{do}(I,\boldsymbol{\xi}_I)}\bigr)_{\mathrm{do}(J,\boldsymbol{\xi}_J)} &= \bigl(\mathcal{M}_{\mathrm{do}(J,\boldsymbol{\xi}_J)}\bigr)_{\mathrm{do}(I,\boldsymbol{\xi}_I)}\\
& = \mathcal{M}_{\mathrm{do}(I\cup J,\boldsymbol{\xi}_{I\cup J})}.
\end{align*}
\end{lemma}
\begin{proof}
The result follows directly from Definition \ref{def:interventions}.
\end{proof}

The following example illustrates interventions on a CCM.
\begin{example}[continues=ex:price supply demand]
Suppose that the supply of a product, if it is not targeted by an intervention, is determined by a function $f_S$, which takes as input the price of the product $P$ and an exogenous random variable $E$ (e.g.\ cost of production). The system for price, supply, and demand can be represented by an (underspecified) CCM $\mathcal{M} = (\mathbb{R}^3, \Phi, E)$, where $\Phi$ consists of two causal constraints:
\begin{align*}
  &(S-D,\,\, &&0, \,\, &&\{\emptyset, \{D\}, \{S\}, \{D,S\}\}),\\
  &(S - f_S(P, E),\,\,&&0,\,\, &&\{\emptyset, \{D\}, \{P\}, \{D,P\}\}).
\end{align*}
After an intervention on $P$ we get $\mathcal{M}_{\text{do}(P,\xi_P)} = (\mathbb{R}^3, \widetilde{\Phi}, E)$, where the updated set of causal constraints is given by
\begin{align*}
  &(S - f_S(P, E),\,\,&&0,\,\, &&\{\emptyset, \{D\}, \{P\}, \{D,P\}\}),\\
  &(P,\,\, &&\xi_P, \,\, &&\{\emptyset, \{D\}, \{S\}, \{D,S\}\}).
\end{align*}
Note that after an intervention on $P$, there would be no intervention under which the causal constraint $(S-D,\,\,0, \,\, \{\emptyset, \{D\}, \{S\}, \{D,S\}\})$ is still active (not even for the null intervention), so it is discarded from $\widetilde{\Phi}$.
\end{example}

\subsection{FROM SCM TO CCM}
\label{sec:relation to an scm}

Structural equations in SCMs are constraints that are active as long as their corresponding variables are not targeted by interventions. This can be used to demonstrate how, for real-valued SCMs, an equivalent CCM with the same solutions under interventions can be constructed.\footnote{The general case, where variables take value in a standard measurable space, requires an additive structure on the variable domains with a zero-element.} 
\begin{proposition}
\label{lemma:scm vs ccm}
Let $\mathcal{M}^{\text{SCM}}=(\mathbb{R}^p, F ,\boldsymbol{E})$ be a real-valued SCM and $\mathcal{I}=\{1,\ldots,p\}$ an index set. The CCM $\mathcal{M}^{\text{CCM}}=(\mathbb{R}^p, \Phi ,\boldsymbol{E})$ with causal constraints $\Phi$:
\begin{multline*}
\bigl(f_{j}(\boldsymbol{x}_{\text{pa}(j)}, \boldsymbol{e}_{\text{pa}(j)}) - x_j,
  \,\, 0,
  \,\, A_j = \mathcal{P}(\mathcal{I}\backslash \{j\})\bigr), \,\forall j\in \mathcal{I},
\end{multline*}
has the same solutions as $\mathcal{M}^{\text{SCM}}$ under any intervention.
\end{proposition}
\begin{proof}
The result follows from Definitions \ref{def:rv solutions} and \ref{def:interventions}.
\end{proof}

\subsection{EQUILIBRIUM CAUSAL MODELS}
We have seen that SCMs may fail to completely capture the causal semantics of stationary behaviour in dynamical systems. Here we prove that CCMs can always completely represent such causal semantics.

\begin{theorem}
  \label{thm:ccm representation}
	Let $\mathcal{D}$ be a dynamical system such that for all
	$I\subseteq\mathcal{I}$ and all $\boldsymbol{\xi}_I\in\boldsymbol{\mathcal{X}}_I$, $\mathcal{D}_{\mathrm{do}(I,\boldsymbol{\xi}_I)}$ has a
  unique solution of the form (\ref{eq:ODE_solution}).
  Then there exists a CCM $\mathcal{M}(\mathcal{D})$ such that for all
	$I\subseteq\mathcal{I}$ and all $\boldsymbol{\xi}_I\in\boldsymbol{\mathcal{X}}_I$: 
	\begin{compactitem}
		\item the equilibrium solutions of $\mathcal{D}_{\mathrm{do}(I,\boldsymbol{\xi}_I)}$ 
		coincide with the solutions of $\big(\mathcal{M}(\mathcal{D})\big)_{\mathrm{do}(I,\boldsymbol{\xi}_I)}$\,,
		\item the following diagram commutes:\\
		\begin{minipage}{\linewidth}
			\centering
			\begin{tikzpicture}
			\centering
			\matrix (m) [matrix of math nodes,row sep=1em,column sep=3em,minimum width=2em]
			{
				\mathcal{D}
				&\mathcal{M}(\mathcal{D})\\
				\mathcal{D}_{\mathrm{do}(I,\boldsymbol{\xi}_I)}
				&\big(\mathcal{M}(\mathcal{D})\big)_{\mathrm{do}(I,\boldsymbol{\xi}_I).}\\};
			\path[-stealth]
			(m-1-1) edge[|->] node [left] {} (m-2-1)
			edge [|->] node [below] {} (m-1-2)
			(m-2-1.east|-m-2-2) edge[|->] node [below] {}
			node [above] {} (m-2-2)
			(m-1-2) edge[|->] node [right] {} (m-2-2);
			\end{tikzpicture}
		\end{minipage}
	\end{compactitem}
\end{theorem}
\vspace{-\baselineskip}
\begin{proof}
  By assumption, the intervened system
	$\mathcal{D}_{\mathrm{do}(I,\boldsymbol{\xi}_I)}$ has a unique solution
	$\boldsymbol{X}_t(\boldsymbol{\xi}_I,\boldsymbol{e}_{\mathcal{I}\setminus I}):=\boldsymbol{X}(t,(\boldsymbol{\xi}_I,\boldsymbol{e}_{\mathcal{I}\setminus I}))$
	which is measurable in $(\boldsymbol{\xi}_I,\boldsymbol{e}_{\mathcal{I}\setminus
		I})$ for all $t$. 
  For $I \subseteq \mathcal{I}$, let $\boldsymbol{C}_I := \{(\boldsymbol{\xi}_I,\boldsymbol{e}_{\mathcal{I}\backslash I}) \in \mathbb{R}^{|I|}\times \mathbb{R}^{|\mathcal{I}\backslash I|}  : \boldsymbol{X}_t(\boldsymbol{\xi}_I,\boldsymbol{e}_{\mathcal{I}\setminus I})\text{ converges for $t\to\infty$}\}$. 
	Consider the measurable function
  $\boldsymbol{g}^I:\mathbb{R}^{\mathcal{I}}\times\mathbb{R}^{\mathcal{I}\backslash I}\to \mathbb{R}^{\mathcal{I}}$ defined by
	{\small
		\begin{align*}
		\boldsymbol{g}^I(\boldsymbol{x},\boldsymbol{e}_{\mathcal{I}\backslash I})
		&:=\boldsymbol{X}^*\big((\boldsymbol{x}_I,\boldsymbol{e}_{\mathcal{I}\setminus I})\big)\boldsymbol{1}_{\boldsymbol{C}_I}\big((\boldsymbol{x}_I,\boldsymbol{e}_{\mathcal{I}\backslash I})\big)\\
		&+
		(\boldsymbol{x} + \boldsymbol{1})\left(\boldsymbol{1} - \boldsymbol{1}_{\boldsymbol{C}_I}\big((\boldsymbol{x}_I,\boldsymbol{e}_{\mathcal{I}\backslash I})\big)\right) 
		- 
		\boldsymbol{x}.
		\end{align*}}%
	The constraint $\boldsymbol{g}^I(\boldsymbol{x},\boldsymbol{e}_{\mathcal{I}\backslash I}) = \boldsymbol{0}$ gives a contradiction if and only if $(\boldsymbol{x}_I,\boldsymbol{e}_{\mathcal{I}\backslash I}) \notin \boldsymbol{C}_I$, and reduces to the equation $\boldsymbol{x} = \boldsymbol{X}^*((\boldsymbol{x}_I,\boldsymbol{e}_{\mathcal{I}\backslash I}))$ otherwise. Therefore, the equilibrium solutions of $\mathcal{D}_{\mathrm{do}(I,\boldsymbol{\xi}_I)}$ coincide with the solutions of the equation $\boldsymbol{g}^I(\boldsymbol{x},\boldsymbol{e}_{\mathcal{I}\backslash I}) = \boldsymbol{0}$. The CCM $\mathcal{M}(\mathcal{D}) := (\boldsymbol{\mathcal{X}}, \Phi, \boldsymbol{E})$ with $\Phi=\{(\boldsymbol{g}^{I}, \boldsymbol{0}, A^I = \{I\}): I\subseteq\mathcal{I}\}$ satisfies the properties of the theorem by construction.
\end{proof}

Theorem \ref{thm:ccm representation} proves that a CCM representation always exists that completely characterizes the causal semantics of a dynamical system at equilibrium. Although we construct a CCM in the proof of the theorem, it does not give a parsimonious representation of the system.\footnote{Interestingly, the CCM construction in the proof of Theorem \ref{thm:ccm representation} can be applied to dynamical systems at finite time $t$.} In the next section, we will outline an intuitive and more convenient construction method in the context of ODEs.

\section{FROM ODE TO CCM}
We consider how and when parsimonious CCM representations can be derived from ODEs and initial conditions in a dynamical system. We demonstrate how causal constraints completely capture the stationary behavior of the basic enzyme reaction and how, unlike SCMs, they are able to correctly represent non-convergence.

\subsection{CAUSAL CONSTRAINTS FROM DIFFERENTIAL EQUATIONS}
When modeling the stationary behavior of a system of ODEs, setting the time-derivatives equal to zero constrains the solution space of the equilibrium model to the fixed points of the system. A CCM allows us to interpret such constraints as \emph{causal} by explicitly specifying under which interventions they put constraints on the equilibrium solutions of the system.

\begin{example}[label=ex:ber]
For the basic enzyme reaction, some of the causal constraints are obtained by setting the time derivatives of the four variables of the system in equations (\refeq{eq:ber ode first}) to (\refeq{eq:ber ode last}) to zero. The resulting equations constrain the solutions of the system as long as the corresponding variables are not targeted by an intervention. This leads to the causal constraints in equations (\refeq{eq:ber first constraint}) to (\refeq{eq:ber last scm constraint}) below,
\begin{align}
\label{eq:ber first constraint}
&(k_0 + k_{-1}C^* - k_1 S^*E^*, && 0, && \mathcal{P}(\mathcal{I}\backslash \{S\})),\\
\label{eq:ber c constraint}
&(k_1 S^*E^* - (k_{-1}+k_2)C^*, && 0, && \mathcal{P}(\mathcal{I}\backslash \{C\})),\\
\label{eq:ber e constraint}
&(-k_1 S^*E^* + (k_{-1}+k_2) C^*, && 0, && \mathcal{P}(\mathcal{I}\backslash \{E\})),\\
\label{eq:ber last scm constraint}
&(k_2 C^* - k_3P^*, && 0, && \mathcal{P}(\mathcal{I}\backslash \{P\})),
\end{align}
with $\mathcal{I}$ an index set for $(S,C,E,P)$. At this stage, the CCM is equivalent to the underspecified SCM of the dynamical system (see also section 2.3). In the next section we will proceed by adding more causal constraints.
\end{example}

\begin{example}[label=ex:lv]
The Lotka-Volterra model \citep{Murray2002} is a set of differential equations that is often used to describe the dynamics of a system where prey (e.g.\ deer) and predators (e.g.\ wolves), $X_1$ and $X_2$, interact. The dynamics of the biological model are given by
\begin{align}
\label{eq:lv 1}
\dot{X_1} &= X_1(t) (k_{11}-k_{12}X_2(t)),\\
\label{eq:lv 2}
\dot{X_2} &= -X_2(t)(k_{22}-k_{21}X_1(t)),
\end{align}
with initial values $X_1(0)>0, X_2(0)>0$ and strictly positive rate parameters. The system has two fixed points $(X_1^*,X_2^*)=(0,0)$ and $(X_1^*,X_2^*)=(k_{22}/k_{21}, k_{11}/k_{12})$, which can be represented either by causal constraints,
\begin{align}
\label{eq:lv constraint x1}
&(X_1^*(k_{11}-k_{12}X^*_2),0,\{\emptyset,\{2\}\}),\\
\label{eq:lv constraint x2}
&(X_2^*(k_{22}-k_{21}X^*_1),0,\{\emptyset,\{1\}\}),
\end{align}
or (equivalently) by structural equations:
\begin{align*}
X^*_1 &= X^*_1 + X^*_1(k_{11}-k_{12}X^*_2),\\
X^*_2 &= X^*_2 - X^*_2(k_{22}-k_{21}X^*_1).
\end{align*}
These (structural) equations do not describe the stable steady state behavior of the model, because the system displays undamped oscillations around the positive fixed point, as was pointed out by \citet{Murray2002, Mooij2013}. In the next section we proceed by adding additional relevant constraints to the CCM.
\end{example}

\subsection{CAUSAL CONSTRAINTS FROM CONSTANTS OF MOTION}
For dynamical systems that admit a constant of motion (i.e.\ a conserved quantity), the trajectories of its solutions are confined to a space that is constrained by its initial conditions. Hence, the solutions for the equilibrium must be similarly constrained. In a CCM we interpret these constraints as causal by specifying under which interventions they constrain the solution space.

\begin{example}[continues=ex:ber]
For the basic enzyme reaction, we include the conservation law that results from the linear dependence between the time derivative of the free enzyme $E$ and the complex $C$ in equation (\refeq{eq:ber constant of motion}). Since this relation holds as long as the `cycle' between $C$ and $E$ is not broken, we obtain the following causal constraint
\begin{align}
\label{eq:ber constant constraint}
(C^*+E^* - (c_0+e_0), && 0, && \mathcal{P}(\mathcal{I}\backslash \{C,E\})).
\end{align}
Another conservation law appeared after intervention on the variable $C$. The resulting conservation law $S(t)-E(t)=s_0-e_0$ applies as long as the `cycle' between $S$ and $E$ is not broken by another intervention on the system. This leads to the final causal constraint:
\begin{align}
\label{eq:ber last constraint}
(S^*-E^* - (s_0-e_0), && 0, && \{\{C\},\{C,P\}\}).
\end{align}
Let $\Phi$ be the set of causal constraints in (\ref{eq:ber first constraint}) to (\ref{eq:ber last scm constraint}) and (\ref{eq:ber constant constraint}) to (\ref{eq:ber last constraint}). In Section \ref{sec:ber equilibrium states} we showed that the active constraints in $\Phi$ have a unique solution under any intervention. If $\boldsymbol{E}=(s_0,e_0,c_0,p_0)$ is a set of exogenous random variables then the CCM $\mathcal{M}=(\mathbb{R}_{>0}^4, \Phi, \boldsymbol{E})$ completely captures the stationary behaviour of the basic enzyme reaction.
\end{example}

\begin{remark}
Interestingly, if we treat $C$ as a latent endogenous variable that cannot be intervened upon, the equilibrium to which the dynamics of the basic enzyme reaction converges can be described by the following marginal CCM (see supplementary material for details):
\begin{align*}
&\textstyle\frac{k_0 + k_{-1}\frac{k_0}{k_2}}{k_1 E^*}-S^*, && 0, && \mathcal{P}(\mathcal{I}'\backslash \{S\}),\\
&\textstyle\frac{(k_{-1}+k_2)(c_0 + e_0)}{k_{-1}+k_2+k_1S^*}-E^*, && 0, && \mathcal{P}(\mathcal{I}'\backslash \{E\}),\\
&\textstyle\frac{k_2}{k_3}\frac{k_1 S^*E^*}{k_{-1}+k_2}-P^*, && 0, && \mathcal{P}(\mathcal{I}'\backslash \{P\}),
\end{align*}
where $\mathcal{I}'$ is an index set for $\{S,E,P\}$. From Proposition \ref{lemma:scm vs ccm} it can be seen that there exists an equivalent SCM that does completely capture the causal semantics of $S,E,$ and $P$, as long as one does not intervene on $C$.
\end{remark}

\begin{example}[continues=ex:lv]
The Lotka-Volterra model provides an example of a system that admits a non-linear conservation law:
{\small
\begin{align}
\label{eq:lv constant of motion}
&k_{21}X_1 + k_{22} \log(X_1) - k_{12}X_2 + k_{11} \log(X_2)=\\
&-k_{21}X_1(0) + k_{22} \log(X_1(0)) - k_{12}X_2(0) + k_{11} \log(X_2(0)),\notag
\end{align}}%
which represents a constraint that is only active in the observational setting. If the system would converge to an equilibrium $(X_1^*,X_2^*)$ the causal constraints derived from the differential equations should hold simultaneously. These constraints are only satisfied when the system starts out in one of the fixed points (e.g.\ $(X_1(0),X_2(0))=(k_{22}/k_{21}, k_{11}/k_{12})$). Otherwise the dynamical system exhibits steady-state oscillations and the set of causal constraints has no solution.

A complete causal description can be obtained by adding the following two causal constraints:
\begin{align}
\label{eq:lv intervene 2}
&(X^*_1 - X_1(0)\boldsymbol{1}_{\{k_{11}-k_{12}X^*_2\geq 0\}}, && 0, &&\{\{2\}\}),\\
\label{eq:lv intervene 1}
&(X^*_2 - X_2(0)\boldsymbol{1}_{\{k_{22}-k_{21}X^*_1\leq 0\}}, && 0, &&\{\{1\}\}).
\end{align}
Addition of the causal constraint in equation (\ref{eq:lv intervene 2}) ensures that after an intervention on the amount of predators $X_2$: a) the prey $X_1$ goes extinct when there are too many predators b) the model has no solution if there are too few predators and c) the amount of prey is constant if the amount of predators is exactly right. The causal constraint in equation (\ref{eq:lv intervene 1}) can be interpreted similarly. Together, the causal constraints in equations (\ref{eq:lv constraint x1}), (\ref{eq:lv constraint x2}), (\ref{eq:lv constant of motion}), (\ref{eq:lv intervene 2}), and (\ref{eq:lv intervene 1}) capture the stationary behavior of the predator-prey model.\footnote{This can be verified by explicitly calculating the solutions of the model under all interventions.} The SCM on the other hand has the fixed points of the system as a solution and does not predict the non-convergent behavior.
\end{example}

\subsection{CONSTRUCTING CCMs}
Causal constraints (or structural equations) derived from differential equations result in a causal description of the fixed points in a system. For structurally semistable systems the addition of causal constraints derived from constants of motion results in a complete causal description of the system's stationary behavior when the constraints specify the equilibria in terms of initial conditions.

\begin{theorem}
\label{thm:condition to construct ccm}
  Let $\mathcal{D}$ be a dynamical system that converges to a fixed point if it has at least one. Let $\mathcal{M}$ be a CCM constructed from the ODEs and constants of motion in $\mathcal{D}$ for which all solutions, if they exist, are unique up to $\mathbb{P}^{\boldsymbol{E}}$-zero sets. $\mathcal{D}$ converges to an equilibrium $\boldsymbol{X}^*$ if and only if $\boldsymbol{X}^*$ is a solution of $\mathcal{M}$.
\end{theorem}
\begin{proof}
First assume that $\mathcal{D}$ has a fixed point, so that $\mathcal{D}$ converges to an equilibrium $\boldsymbol{X}^*(\boldsymbol{e})$ for almost every $\boldsymbol{e}\in\mathbb{R}^p$. We have that a) $\boldsymbol{X}^*(\boldsymbol{e})$ satisfies the constants of motion in the dynamical system and b) for $\boldsymbol{X}^*(\boldsymbol{e})$ the time-derivatives appearing in the ODEs are equal to zero. Hence if $\mathcal{D}$ converges to $\boldsymbol{X}^*$ then $\boldsymbol{X}^*$ is a solution of $\mathcal{M}$. Since $\mathcal{M}$ has no more than one solution (up to zero sets), the reverse statement is also true. Now assume that $\mathcal{D}$ has no fixed point. In that case $\mathcal{M}$ has no solutions, and $\mathcal{D}$ cannot converge to an equilibrium.
\end{proof}

\begin{corollary}
\label{cor:semistable to construct ccm}
Let $\mathcal{D}$ be structurally semistable and $\mathcal{M}$ a CCM constructed from the ODEs and constants of motion in $\mathcal{D}$ for which under any intervention, all solutions, if they exist, are unique up to $\mathbb{P}^{\boldsymbol{E}}$-zero sets. Then for all $I\subseteq\mathcal{I}$ and $\boldsymbol{\xi}_I\in\mathbb{R}^{|I|}$: $\mathcal{D}_{\mathrm{do}(I,\boldsymbol{\xi}_I)}$ converges to an equilibrium $\boldsymbol{X}^*(I,\boldsymbol{\xi}_I)$ iff $\boldsymbol{X}^*(I,\boldsymbol{\xi}_I)$ is a solution of $\mathcal{M}_{\mathrm{do}(I,\boldsymbol{\xi}_I)}$.
\end{corollary}
\begin{proof}
  If $\mathcal{M}_{\mathrm{do}(I,\boldsymbol{\xi}_I)}$ has a solution then $\mathcal{D}_{\mathrm{do}(I,\boldsymbol{\xi}_I)}$ has a fixed point with $\boldsymbol{x}^*_I=\boldsymbol{\xi_I}$ and it converges because $\mathcal{D}$ is structurally semistable. If $\mathcal{M}_{\mathrm{do}(I,\boldsymbol{\xi}_I)}$ has no solution then $\mathcal{D}_{\mathrm{do}(I,\boldsymbol{\xi}_I)}$ does not converge to a fixed point. The result follows from Theorem \ref{thm:condition to construct ccm} and Definition \ref{def:interventions}.
\end{proof}
\vspace{-6pt}
The basic enzyme reaction in Example \ref{ex:ber} is structurally semistable, while the Lotka-Volterra model in Example \ref{ex:lv} is not. Corollary \ref{cor:semistable to construct ccm} tells us that for structurally semistable systems, if a CCM constructed from ODEs and constants of motions has at most one solution under any intervention, then the CCM completely captures the causal semantics of the stationary behaviour of the system.

\section{FUNCTIONAL LAWS}
CCMs can also represent \emph{functional laws}, which are relations between variables that are invariant under \emph{all} interventions. Causal constraints allow one to explicitly state under which interventions a constraint is active. Therefore a CCM never admits a solution that violates the functional law, where an SCM would.

\begin{example}
\label{ex:igl}
It is well-known that the pressure $P$ and temperature $T$ for $N$ particles of an ideal gas in a fixed volume $V$ are related by the ideal gas law. In absence of any knowledge about the environment, this system can be represented by the (underspecified) CCM $\mathcal{M}=(\mathbb{R}^2, \{(PV - Nk_BT, 0, \mathcal{P}(\mathcal{I}))\}, \mathbb{P}_{\emptyset})$, where $k_B$ is Boltzmann's constant, and
$\mathcal{I}$ is an index set for the variables $(P,T)$ in the system. If we were to describe the same system using an SCM, then we would need two copies of this causal constraint as structural equations:
\begin{align*}
P = \textstyle\frac{Nk_BT}{V}, \qquad \qquad	T = \textstyle\frac{PV}{Nk_B}.
\end{align*}
Indeed, considering interventions on one of the variables leaves no choice for the structural equation of the other one. Furthermore, a simultaneous intervention on $P$ and $T$ always has a solution in the SCM representation, even when this means that the ideal gas law is violated. The CCM representation typically does not have a solution under such an intervention (unless the target values satisfy the ideal gas law constraint). Therefore, the CCM representation of functional laws like the ideal gas law is more parsimonious and more natural than any SCM representation can be.
\end{example}

A functional law can be any relation that is invariant under all interventions. For example, a transformation of a (set of) variables to another (set of) variables describing the same system can also be modeled as a functional law.

\begin{example}
Let $\mathcal{I}$ be an index set of $(T,V,O)$. Suppose that the viscosity $T$ of a salad dressing, consisting of a certain amount of oil $O$ and a certain amount of vinegar $V$ is determined by a causal constraint $\phi=(f,0,\mathcal{P}(\mathcal{I}\backslash\{T\}))$ where $f$ is a function depending on the amount of oil and vinegar. By adding causal constraints
\begin{align*}
&(O_r - O/(O+V) ,\quad &&0,\quad &&\mathcal{P}({\mathcal{I}})),\\
&(V_r - V/(O+V) ,\quad &&0,\quad &&\mathcal{P}({\mathcal{I}})),
\end{align*}
a CCM allows us to have the relative amounts of oil and vinegar $O_r$ and $V_r$ in the model without running into logical contradictions.
\end{example}

\section{CONCLUSION}
While Structural Causal Models (SCMs) form a very popular modeling framework in many applied sciences, we have shown that they are neither powerful enough to model the rich equilibrium behavior of simple dynamical systems such as the basic enzyme reaction, nor simple functional laws of nature like the ideal gas law. This raises the question whether the common starting point in causal discovery---that the data-generating process can be modeled with an SCM---is tenable in certain application domains, for example, for biochemical systems.

We believe that the examples presented in this paper form a compelling motivation to extend the common causal modeling framework to potentially broaden the impact of causal modeling in dynamical systems. In this work, we introduced Causal Constraints Models (CCMs). We showed how they can be `constructed' from differential equations and initial conditons and proved that they can completely capture the causal semantics of functional laws and stationary behavior in dynamical systems.

One intuitively appealing aspect of SCMs is their graphical interpretation. In contrast, CCMs are not equipped with graphical representations yet. In future work, we plan to investigate graphical representations of the conditional independence structure of CCMs. This will allow us to better understand the causal interpretation of the results of existing causal discovery algorithms.

\paragraph{Acknowledgements}
This work was supported by the ERC under the European Union's Horizon 2020 research and innovation programme (grant agreement 639466) and by NWO (VIDI grant 639.072.410).

\printbibliography[title={References}]

\clearpage
\onecolumn
\appendix

\section*{Supplementary Material}

\section{Basic Enzyme Reaction}
In this section we show the additional results, concerning the basic enzyme reaction, that were discussed in the main paper. First we discuss the fixed points of the basic enzyme reaction. Then we show that the systems converges to its fixed point whenever it exists. Finally, we derive a simple marginal model from the CCM representation of the basic enzyme reaction.

\subsection{Fixed points}
The fixed points of the basic enzyme reaction, for all intervened systems, are given in Table \ref{tab:ber complete solutions}. For any intervention, these are obtained by solving the system of equations that one gets by considering the causal constraints in the CCM in (\ref{eq:ber first constraint}) to (\ref{eq:ber last constraint}) that are active under that specific intervention. That is, we take all equations for which the intervention is in the activation set.

\begin{table}[h]
\centering
\caption{Fixed points of the basic enzyme reaction, where $y=\frac{1}{2}\sqrt{(e_0-s_0)^2 +  4\frac{k_0(k_{-1}+k_2)}{k_1k_2}}$.}
\label{tab:ber complete solutions}
\begin{tabular}{l|cccc}
	intervention & $S$ & $C$           & $E$ & $P$\\
	\hline
	none         & $\frac{k_0 + k_{-1} \frac{k_0}{k_2}}{k_1 (e_0 + c_0 - \frac{k_0}{k_2})}$ & $\frac{k_0}{k_2}$ & $e_0 + c_0 - \frac{k_0}{k_2}$ & $\frac{k_0}{k_3}$ \\
	do($S=s$)  & $s$ & $\frac{k_1 s (e_0 + c_0)}{k_{-1} + k_2 + k_1 s} $ & $\frac{(k_{-1} + k_2)(e_0 + c_0)}{k_{-1} + k_2 + k_1 s} $ & $\frac{k_2}{k_3}\frac{k_1 s (e_0 + c_0)}{k_{-1} + k_2 + k_1 s}$\\
	do($C=c$), $c = \frac{k_0}{k_2}$   & $\frac{(s_0-e_0)}{2} + y$ & $c$ & $\frac{-(s_0-e_0)}{2} + y$ & $\frac{k_2}{k_3} c$ \\
	do($C=c$), $c \ne \frac{k_0}{k_2}$ & $\emptyset$ & $\emptyset$ & $\emptyset$ & $\emptyset$ \\
	do($E=e$) & $\frac{k_0 + k_{-1} \frac{k_0}{k_2}}{k_1 e}$ & $\frac{k_0}{k_2}$ & $e$ & $\frac{k_0}{k_3}$\\
	do($P=p$) &  $\frac{k_0 + k_{-1} \frac{k_0}{k_2}}{k_1 (e_0 + c_0 - \frac{k_0}{k_2})}$ & $\frac{k_0}{k_2}$ & $e_0 + c_0 - \frac{k_0}{k_2}$ & $p$ \\
	do($S=s$, $C=c$) & $s$ & $c$ & $\frac{k_{-1} + k_2}{k_1} \frac{c}{s}$ & $\frac{k_2}{k_3}c$\\
	do($S=s$, $E=e$)  & $s$ & $\frac{k_1}{k_{-1} + k_2} s e$ & $e$ & $\frac{k_2}{k_3} \frac{k_1}{k_{-1} + k_2} s e$\\
	do($S=s$, $P=p$)  & $s$ & $\frac{k_1 s (e_0 + c_0)}{k_{-1} + k_2 + k_1 s} $ & $\frac{(k_{-1} + k_2)(e_0 + c_0)}{k_{-1} + k_2 + k_1 s} $ & $p$\\
	do($C=c$, $E=e$) & $\frac{k_0 + k_{-1} c}{k_1 e}$ & $c$ & $e$ & $\frac{k_2}{k_3}c$ \\
	do($C=c$, $P=p$), $c = \frac{k_0}{k_2}$   & $\frac{(s_0-e_0)}{2} + y$ & $c$ & $\frac{-(s_0-e_0)}{2} + y$ & $p$ \\
	do($C=c$, $P=p$), $c \ne \frac{k_0}{k_2}$ & $\emptyset$ & $\emptyset$ & $\emptyset$ & $\emptyset$ \\
	do($E=e$, $P=p$) & $\frac{k_0 + k_{-1} \frac{k_0}{k_2}}{k_1 e}$ & $\frac{k_0}{k_2}$ & $e$ & $p$\\
	do($S=s$, $C=c$, $E=e$) & $s$ & $c$ & $e$ &$\frac{k_2}{k_3}c$ \\
	do($S=s$, $C=c$, $P=p$)  & $s$ & $c$ & $\frac{k_{-1} + k_2}{k_1} \frac{c}{s}$ & $p$\\
	do($S=s$, $E=e$, $P=p$) & $s$ & $\frac{k_1}{k_{-1} + k_2} s e$ & $e$ & $p$\\
	do($C=c$, $E=e$, $P=p$) & $\frac{k_0 + k_{-1} c}{k_1 e}$ & $c$ & $e$ & $p$ \\
	do($S=s$, $C=c$, $E=e$, $P=p$) & $s$ & $c$ & $e$ & $p$
\end{tabular}
\end{table}

\subsection{Convergence results for the basic enzyme reaction}
In this section, we show that the basic enzyme reaction always converges to its fixed point, as long as it exists. We also show that the intervened basic enzyme reaction has the same property. To prove this result we rely on both explicit calculations and a convergence property of so-called cooperative systems that we obtained from \citet{Belgacem2012}. To prove convergence for the observed system and the system after interventions on $P$ and $E$, we use the latter technique. Convergence to the equilibrium solution after interventions on $S$ and $C$ can be shown by explicit calculation. The convergence results for combinations of interventions can be obtained by a trivial extension of the arguments that were used in the other cases.

\subsubsection{Cooperativity in the basic enzyme reaction}
To show that the basic enzyme reaction converges to a unique equilibrium, if it exists, we first state a result that we obtained from \citet{Belgacem2012}: cooperative systems as in Definition \ref{def:cooperative systems} have the attractive convergence property in Proposition \ref{prop:convergence cooperative systems}.

\begin{definition}
\label{def:cooperative systems}
A system of ODEs $\dot{\boldsymbol{X}}$ is \emph{cooperative} if the Jacobian matrix has non-negative off-diagonal elements, or there exists an integer $k$ such that the Jabobian has $(k\times k)$ and $(n-k)\times(n-k)$ main diagonal matrices with nonnegative off-diagonal entries and the rectangular off-diagonal submatrices have non-positive entries.
\end{definition}

\begin{proposition}
\label{prop:convergence cooperative systems}
Let $\dot{\boldsymbol{X}}=\boldsymbol{f}(\boldsymbol{X})$ be a cooperative system with a fixed point $\boldsymbol{x}^*$. If there exist two points $\boldsymbol{x}_{\text{min}},\boldsymbol{x}_{\text{max}}\in\boldsymbol{\mathcal{X}}$ such that $\boldsymbol{x}_{\text{min}}\leq \boldsymbol{x}^* \leq \boldsymbol{x}_{\text{max}}$ and $\boldsymbol{f}(\boldsymbol{x}_{\text{min}})\geq 0$ and $\boldsymbol{f}(\boldsymbol{x}_{\text{max}})\leq 0$, then the hyperrectangle betweeen $\boldsymbol{x}_{\text{min}}$ and $\boldsymbol{x}_{\text{max}}$ is invariant\footnote{An invariant set is a set with the property that once a trajectory of a dynamical set enters it, it cannot leave.} and for almost all initial conditions inside this rectangle the solution converges to $\boldsymbol{x}^*$.
\end{proposition}

\subsubsection{Convergence of the observed system}
Recall that the dynamics of the basic enzyme reaction are given by
\begin{align}
\label{eq:ber first}
\dot{S}(t) &= k_0-k_{1}S(t)E(t) + k_{-1} C(t),\\
\dot{E}(t) &= -k_1S(t)E(t) + (k_{-1}+k_2)C(t),\\
\dot{C}(t) &= k_1S(t)E(t) - (k_{-1}+k_2)C(t),\\
\label{eq:ber last}
\dot{P}(t) &= k_2C(t) - k_3 P(t),\\
S(0)=s_0, &\quad E(0)=e_0, \quad C(0)=c_0, \quad P(0)=p_0,
\end{align}
where $\boldsymbol{x}_0=(s_0,e_0,c_0,p_0)$ are the \emph{initial conditions} of the system.

The analysis in \citet{Belgacem2012} of the basic enzyme reaction makes use of Proposition \ref{prop:convergence cooperative systems}, but also includes feedback from $P$ to $C$. In this section, we repeat their analysis on our sligthly different model. Note that the arguments given in this section can also be applied to the system where $P$ is intervened upon.

We start by rewriting the system of ODEs in equation (\ref{eq:ber first}) to (\ref{eq:ber last}), by using the fact that $\dot{E}(t)+\dot{C}(t)=0$ so that $E(t)=e_0+c_0-C(t)$:
\begin{align}
\dot{S}(t) &= k_0-k_{1}S(t)(e_0+c_0-C(t)) + k_{-1} C(t),\\
\dot{C}(t) &= k_1S(t)(e_0+c_0-C(t))- (k_{-1}+k_2)C(t),\\
\dot{P}(t) &= k_2C(t) - k_3 P(t).
\end{align}

\paragraph{Cooperativity} The corresponding Jacobian matrix is given by,
\begin{equation}
J(S,C,P) = \begin{pmatrix}
-k_1(e_0+c_0-C(t)) & k_{-1}+k_1S(t) & 0\\
k_1(e_0+c_0-C(t)) & -(k_{-1}+k_2)-k_1S(t) & 0\\
0 & k_2 & -k_3
\end{pmatrix}.
\end{equation}
Since all off-diagonal elements in the Jacobian matrix are nonnegative, the observational system is a cooperative system by Definition \ref{def:cooperative systems}.

\paragraph{Convergence} From Table \ref{tab:ber complete solutions} we find that the observed system has a unique (positive) fixed point as long as $e_0+c_0 > \frac{k_0}{k_2}$. We want to use Proposition \ref{prop:convergence cooperative systems} to show that the system converges to this fixed point, so we need to find $\boldsymbol{x}_{\text{min}}$ and $\boldsymbol{x}_{\text{max}}$ so that all three derivatives are nonnegative and nonpositive respectively.

For $\boldsymbol{x}_{\text{min}}=(0,0,0)$, then $\dot{S}=k_0>0$ and $\dot{C}=\dot{P}=0$ so all derivatives are nonnegative. The upper vertex must be chosen so that all derivative are non-positive:
\begin{align*}
\dot{S} \leq 0 & \iff S \geq \frac{k_0 + k_{-1}C}{k_1(e_0+c_0-C)},\\
\dot{C} \leq 0 & \iff S \geq \frac{(k_{-1}+k_2)C}{k_1(e_0+c_0-C)},\\
\dot{P} \leq 0 & \iff P \geq \frac{k_2}{k_3}C.
\end{align*}
The basic enzyme reaction only has a fixed point as long as $C<e_0+c_0$ (otherwise $\dot{S}(t)>0$). If we let $C$ approach $e_0+c_0$, then the inequality constraints on the derivatives are satisfied as $S$ and $P$ go to infinity. More formally we can choose
\begin{equation*}
\boldsymbol{x}_{\text{max}} = (S=\max\left(\frac{k_0 + k_{-1}C}{k_1(e_0+c_0-C)}, \frac{(k_{-1}+k_2)C}{k_1(e_0+c_0-C)}\right),  C=e_0+c_0-\epsilon, P=\frac{k_2}{k_3}C + \frac{1}{\epsilon}).
\end{equation*}
When $\epsilon$ approaches zero, both $S$ and $P$ go to infinity and all derivatives are nonpositive. Hence, by Proposition \ref{prop:convergence cooperative systems}, the system converges to its fixed point for almost all valid initial values of $S,C,$ and $P$ (for which the fixed point exists).

\subsubsection{Intervention on E}
Similarly, we can also show that the system where $E$ is targeted by an intervention that sets it equal to $e$, converges to the (unique) equilibrium in Table \ref{tab:ber complete solutions}. The intervened system of ODEs is given by
\begin{align*}
\dot{S} &= k_0 -k_1 e S + k_{-1}C,\\
\dot{C} &= k_1 e S - (k_{-1}+k_2) C,\\
\dot{P} &= k_2 C - k_3 P.
\end{align*}
The Jacobian is given by
\begin{equation}
J(S,C,P) = \begin{pmatrix}
-k_1e & k_{-1} & 0\\
k_1 e & -(k_{-1}+k_2) & 0\\
0 & k_2 & -k_3
\end{pmatrix}.
\end{equation}
Since all off-diagonal elements are nonnegative this is a cooperative system by Definition \ref{def:cooperative systems}.

All derivatives are nonnegative at the point $(S,C,P)=(0,0,0)$, and all derivatives are nonpositive at the point $(s,c,p)$ where
\begin{align*}
s &= \text{max}\left(\frac{k_{-1}c + k_0}{k_1 e}, \frac{(k_{-1}+k_2)c}{k_1 e}\right),\\
p &= \frac{k_2}{k_3} c,
\end{align*}
where $c\to\infty$. We then apply Proposition \ref{prop:convergence cooperative systems} to show that the intervened system converges to the equilibrium value from all valid initial values.

\subsubsection{Intervention on S}
We show that the system converges to the equilibrium solution after an intervention on $S$ by explicit calculation. The intervened system of ODEs is given by
\begin{align*}
\dot{S}(t) &= 0,\\
\dot{E}(t) &= -k_1sE(t) + (k_{-1}+k_2)C(t),\\
\dot{C}(t) &= k_1sE(t) - (k_{-1}+k_2)C(t),\\
\dot{P}(t) &= k_2C(t) - k_3 P(t).
\end{align*}
Since $\dot{C}(t)+\dot{E}(t)=0$, we can write $E(t)=e_0 + c_0 - C(t)$, resulting in the following differential equation
\begin{align}
\dot{C}(t) &= k_1s (e_0 + c_0 - C(t)) - (k_{-1}+k_2)C(t),\\
&= -(k_1 s + k_{-1}+k_2) C(t) + k_1s(e_0+c_0).
\end{align}
We take the limit $t\to\infty$ of the solution to the initial value problem to obtain
\begin{equation}
C^* = \lim_{t\to\infty} \frac{k_1s (e_0 + c_0)}{(k_1 s + k_1+k_2)} + e^{-(k_1 s + k_{-1}+k_2)t} = \frac{k_1s (e_0 + c_0)}{(k_1 s + k_{-1}+k_2)}.
\end{equation}
The result for $E$ follows from the fact that $E(t)=e_0 + c_0 - C(t)$. The result for $P$ follows by explicitly solving the differential equation and taking the limit $t\to\infty$.

\subsubsection{Intervention on C}
There is no equilibrium solution when the intervention targeting $C$ does not have value $\frac{k_0}{k_2}$, as can be seen from Table \ref{tab:ber complete solutions}. To show that the system converges when the equilibrium solution exists, we can explicitly solve the initial value problem and take the limit $t\to\infty$. The intervened system of ODEs after an intervention $\text{do}(C=\frac{k_0}{k_2})$ is given by
\begin{align*}
\dot{S}(t) &= -k_1S(t)E(t) + (k_{-1}+k_2)\frac{k_0}{k_2} = -k_1S(t)E(t) + k,\\
\dot{E}(t) &= -k_1S(t)E(t) + (k_{-1}+k_2)\frac{k_0}{k_2} = -k_1S(t)E(t) + k ,\\
\dot{C}(t) &= 0,\\
\dot{P}(t) &= k_0 - k_3 P(t),
\end{align*}
where we set $k=(k_{-1}+k_2)\frac{k_0}{k_2}$ for brevity.

The initival value problem for $P$ can be solved explicitly, and by taking the limit $t\to\infty$ we obtain
\begin{equation*}
P^* = \lim_{t\to\infty} P(t) = \lim_{t\to\infty}\frac{k_0}{k_3} + c\cdot e^{-k_3 t} = \frac{k_0}{k_3},
\end{equation*}
which is the same as the equilibrium solution in Table \ref{tab:ber complete solutions}.

The solution for $S$ is more involved. First we substitute $E(t)=S(t) - (s_0-e_0)$ (since $\dot{S}(t)-\dot{E}(t)=0$) which gives us the following differential equation
\begin{equation*}
\dot{S}(t) = -k_1S(t)(S(t) - (s_0-e_0)) + k = -k_1 S(t)^2 + (s_0-e_0) k_1 S(t) +k.
\end{equation*}

To solve this differential equation we first divide both sides by $(-k_1(S(t))^2 + (s_0-e_0)k_1S(t) + k)$, and integrate both sides with respect to $t$,
\begin{align}
\int \frac{dS(t)/dt}{-k_1 S(t)^2 + (s_0-e_0) k_1 S(t) +k} dt &= \int 1 dt\\
\label{eq:separated equation}
\int \frac{dS(t)}{-k_1 S(t)^2 + (s_0-e_0) k_1 S(t) +k} &= (t+c)
\end{align}

To evaluate the left-hand side of this equation we want to apply the following standard integral:
\begin{equation}
\label{eq:standard integral}
\int \frac{1}{ax^2 + bx +c}dx = \begin{cases} - \frac{2}{\sqrt{b^2-4ac}} \tanh^{-1}\left(\frac{2ax+b}{\sqrt{b^2-4ac}}\right) + C,\quad \text{if } |2ax+b|<\sqrt{b^2-4ac},\\
-\frac{2}{\sqrt{b^2-4ac}} \coth^{-1}\left(\frac{2ax+b}{\sqrt{b^2-4ac}}\right) + C, \quad \text{else}.
\end{cases}
\end{equation}
for $b^2-4ac>0$. We first check the condition:
\begin{equation*}
b^2-4ac = (s_0-e_0)^2k_1^2 + 4k_1k >0.
\end{equation*}
We now take the first solution to the standard integral (the second solution gives the same limiting result for $S$, as we will see later on). We apply the first solution in (\ref{eq:standard integral}) to (\ref{eq:separated equation}) to obtain
\begin{align}
\frac{2 \tanh^{-1}\left( \frac{2k_1 S(t) - (s_0-e_0) k_1}{\sqrt{4k_1 k +(s_0-e_0)^2 k_1^2 }} \right) }{\sqrt{4k_1 k +(s_0-e_0)^2 k_1^2 }} &= t + c\\
\tanh^{-1}\left( \frac{2k_1 S(t) - (s_0-e_0) k_1}{\sqrt{4k_1 k +(s_0-e_0)^2 k_1^2 }} \right) &= \frac{1}{2}(t + c)\sqrt{4k_1 k +(s_0-e_0)^2 k_1^2 }\\
\label{eq:solve for s}
\frac{2k_1 S(t) - (s_0-e_0) k_1}{\sqrt{4k_1 k +(s_0-e_0)^2 k_1^2 }} &= \tanh \left(\frac{1}{2}(t + c)\sqrt{4k_1 k +(s_0-e_0)^2 k_1^2 }\right),
\end{align}
Solving (\ref{eq:solve for s}) for $S$ gives,
\begin{align*}
S(t) = \frac{1}{2k_1} \left(\tanh \left(\frac{1}{2}(t + c)\sqrt{4k_1 k +(s_0-e_0)^2 k_1^2 }\right) \sqrt{4k_1 k +(s_0-e_0)^2 k_1^2 } + k_1(s_0-e_0)\right).
\end{align*}
By taking the limit $t\to\infty$, plugging in $k=(k_{-1}+k_2)\frac{k_0}{k_2}$, and rewriting we obtain the equilibrium solution in Table \ref{tab:ber complete solutions}:
\begin{align*}
\lim_{t\to\infty} S(t) &= \frac{k_1(s_0-e_0) + \sqrt{4k_1 k +(s_0-e_0)^2 k_1^2 }}{2k_1}\\
&= \frac{k_1(s_0-e_0) + \sqrt{4k_1 (k_{-1}+k_2)\frac{k_0}{k_2} +(s_0-e_0)^2 k_1^2 }}{2k_1}\\
&= \frac{1}{2}\left((s_0-e_0) + \sqrt{(s_0-e_0)^2 + 4 \frac{ k_0 (k_{-1}+ k_2) }{k_1k_2}}\right).
\end{align*}
Note that if we take the second solution to the standard integral in (\ref{eq:standard integral}), then we would have ended up with the same solution for $S(t)$ with $\tanh$ replaced by $\coth$, but the limit $\lim_{t\to\infty}S(t)$ would still be the same.

The solution for $E$ follows from the fact that $E(t)=S(t) - (s_0-e_0)$. The solutions for all joint interventions were found by combining the arguments that were given for the single interventions.

\subsection{Marginal model}
In the paper we presented a marginal model for the basic enzyme reaction. Here we show how it can be derived from the causal constraints in the CCM, which are given by
\begin{align}
\label{eq:ber first constraintb}
k_0 + k_{-1}C - k_1 SE&=0, && \mathcal{P}(\mathcal{I}\backslash\{S\}),\\
\label{eq:ber c constraintb}
k_1 SE - (k_{-1}+k_2)C &=0, && \mathcal{P}(\mathcal{I}\backslash\{C\}),\\
\label{eq:ber e constraintb}
-k_1 SE + (k_{-1}+k_2) C &=0, && \mathcal{P}(\mathcal{I} \backslash \{E\}),\\
\label{eq:ber last scm constraintb}
k_2 C - k_3P &=0, && \mathcal{P}(\mathcal{I} \backslash \{P\}),\\
\label{eq:ber constant constraintb}
C+E - (c_0+e_0) &= 0, && \mathcal{P}(\mathcal{I}\backslash \{C,E\}),\\
\label{eq:ber last constraintb}
S-E - (s_0-e_0) &= 0, && \{\{C\},\{C,P\}\}.
\end{align}

We obtain the marginal model as follows:
\begin{enumerate}
\item Reduce the number of variables that can be targeted by an intervention: $\mathcal{I}'=\{S,E,P\}$.
\item Rewrite the causal constraint in (\ref{eq:ber c constraintb}) to $C=\frac{k_1 SE}{k_{-1}+k_2}$. Note that this equation holds under any intervention in $\mathcal{P}(\mathcal{I}')=\mathcal{P}(\mathcal{I}\backslash \{C\})$. Then substitute this expression for $C$ into equation (\ref{eq:ber first constraintb}) to obtain
\begin{align*}
\frac{k_0 + k_{-1}\frac{k_0}{k_2}}{k_1 E}-S &= 0, && \mathcal{P}(\mathcal{I}'\backslash \{S\}),
\end{align*}
where the activation set of the causal constraint is given by the intersection $\mathcal{P}(\mathcal{I}\backslash \{S\})\cap \mathcal{P}(\mathcal{I}')$. Then substitute this expresion for $C$ into equation (\ref{eq:ber last scm constraintb}) to obtain
\begin{align*}
\frac{k_2}{k_3}\frac{k_1 SE}{k_{-1}+k_2}-P&=0, && \mathcal{P}(\mathcal{I}'\backslash \{P\}),
\end{align*}
where the activation set of the causal constraint is given by the intersection $\mathcal{P}(\mathcal{I}\backslash \{P\})\cap \mathcal{P}(\mathcal{I}')$.
\item Rewrite the causal constraint in (\ref{eq:ber constant constraintb}) to $C=e_0+c_0-E$ and note that this equation holds under interventions in $\mathcal{P}(\mathcal{I}'\backslash \{E\})$. Then substitute this expression for $C$ into equation (\ref{eq:ber e constraintb}) to obtain
\begin{align*}
\frac{(k_{-1}+k_2)(c_0 + e_0)}{k_{-1}+k_2+k_1S}-E&=0, && \mathcal{P}(\mathcal{I}'\backslash \{E\}),
\end{align*}
where the activation set of the causal constraint is given by the intersection $\mathcal{P}(\mathcal{I}\backslash \{C,E\})\cap \mathcal{P}(\mathcal{I}'\backslash \{E\})$.
\end{enumerate}

This procedure results in the following marginal model
\begin{align*}
\frac{k_0 + k_{-1}\frac{k_0}{k_2}}{k_1 E}-S &= 0, && \mathcal{P}(\mathcal{I}'\backslash \{S\}),\\
\frac{(k_{-1}+k_2)(c_0 + e_0)}{k_{-1}+k_2+k_1S}-E&=0, && \mathcal{P}(\mathcal{I}'\backslash \{E\}),\\
\frac{k_2}{k_3}\frac{k_1 SE}{k_{-1}+k_2}-P&=0, && \mathcal{P}(\mathcal{I}'\backslash \{P\}).
\end{align*}
Because we kept track of the interventions under which each equation is active when we substituted $C$ into the equations of other causal constraints, we preserved the causal structure of the model. That is, the marginal CCM model has the same solutions as the original CCM under interventions in $\mathcal{P}(\mathcal{I}')$.

\end{document}